\documentclass{article}

\usepackage{microtype}
\usepackage{graphicx}
\usepackage{subfigure}
\usepackage{booktabs} 

\usepackage{hyperref}

\usepackage[accepted]{icml2024}

\usepackage{amsmath}
\usepackage{amssymb}
\usepackage{mathtools}
\usepackage{amsthm}

\usepackage{listings} 

\definecolor{codegreen}{rgb}{0,0.6,0}
\definecolor{codegray}{rgb}{0.5,0.5,0.5}
\definecolor{codepurple}{rgb}{0.58,0,0.82}
\definecolor{backcolour}{rgb}{0.95,0.95,0.92}

\lstdefinestyle{mystyle}{
    backgroundcolor=\color{backcolour},   
    commentstyle=\color{codegreen},
    keywordstyle=\color{magenta},
    numberstyle=\tiny\color{codegray},
    stringstyle=\color{codepurple},
    basicstyle=\ttfamily\footnotesize,
    breakatwhitespace=false,         
    breaklines=true,                 
    captionpos=b,                    
    keepspaces=true,                 
    numbers=left,                    
    numbersep=5pt,                  
    showspaces=false,                
    showstringspaces=false,
    showtabs=false,                  
    tabsize=2
}

\usepackage[capitalize,noabbrev]{cleveref}

\theoremstyle{plain}
\newtheorem{theorem}{Theorem}[section]

\newtheorem{lemma}[theorem]{Lemma}

\theoremstyle{definition}
\newtheorem{definition}[theorem]{Definition}
\newtheorem{assumption}[theorem]{Assumption}
\theoremstyle{remark}

\usepackage[textsize=tiny]{todonotes}

\usepackage{enumitem}

\icmltitlerunning{QuantAgent: Seeking Holy Grail in Trading by Self-Improving Large Language Model}

\begin{document}

\twocolumn[
\icmltitle{QuantAgent: Seeking Holy Grail in Trading by Self-Improving Large Language Model}



\icmlsetsymbol{equal}{*}

\begin{icmlauthorlist}
\icmlauthor{Saizhuo Wang}{equal,hkust,idea}
\icmlauthor{Hang Yuan}{equal,hkust,idea}
\icmlauthor{Lionel M. Ni}{hkust-gz,hkust}
\icmlauthor{Jian Guo}{idea}
\end{icmlauthorlist}

\icmlaffiliation{hkust}{The Hong Kong University of Science and Technology}
\icmlaffiliation{hkust-gz}{The Hong Kong University of Science and Technology (Guangzhou)}
\icmlaffiliation{idea}{IDEA Research}

\icmlcorrespondingauthor{Jian Guo}{guojian@idea.edu.cn}

\icmlkeywords{Machine Learning, ICML}

\vskip 0.3in
]



\printAffiliationsAndNotice{\icmlEqualContribution} 

\newcommand{\figureOverview}{
\begin{figure*}[!t]
    \centering
    \includegraphics[width=\textwidth]{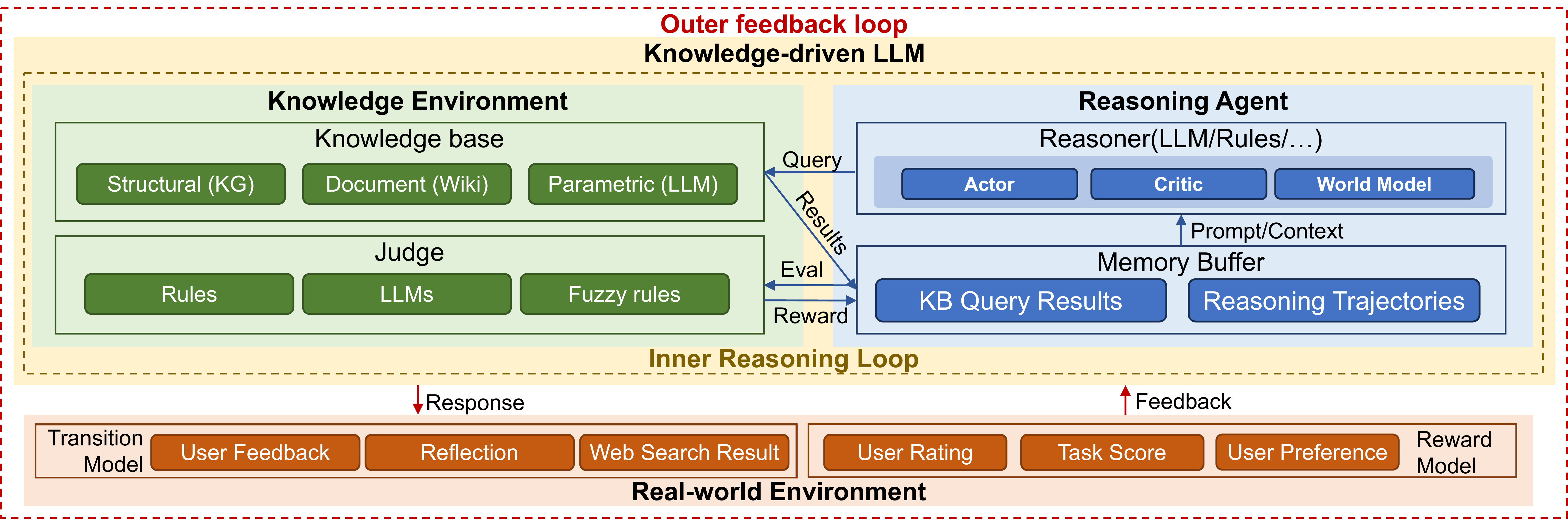}
    \caption{Overview of the two-level loops in AlphaAgent.}
    \label{fig:overview}
\end{figure*}
}

\newcommand{\figureInnerLoop}{
\begin{figure*}[!t]
    \centering
    \includegraphics[width=\textwidth]{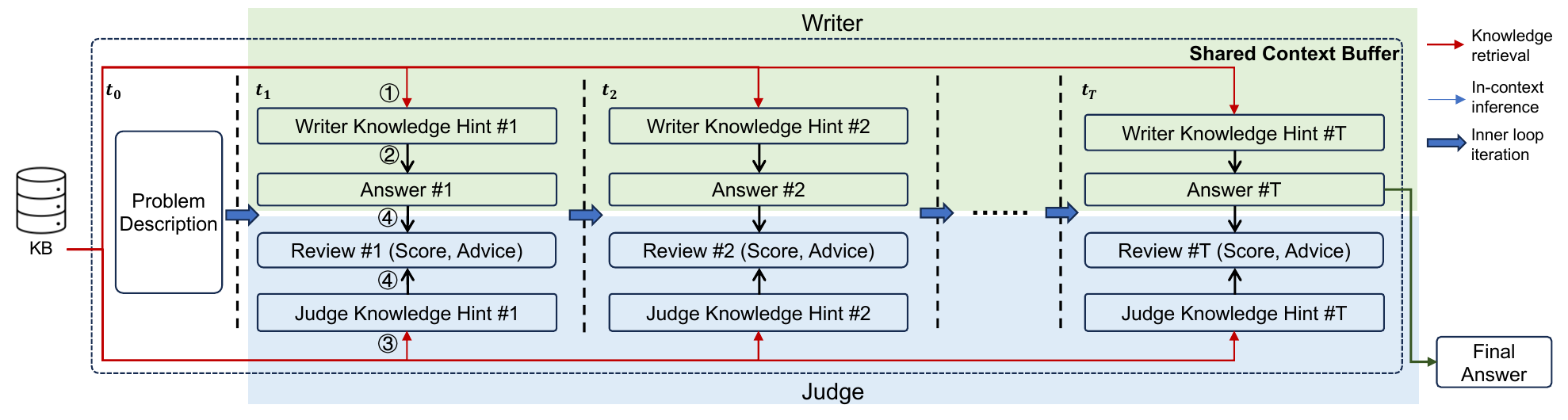}
    \caption{The inner reasoning loop.}
    \label{fig:inner_loop}
\end{figure*}
}

\newcommand{\figureMethod}{
\begin{figure*}[!t]
    \centering
    \includegraphics[width=\textwidth]{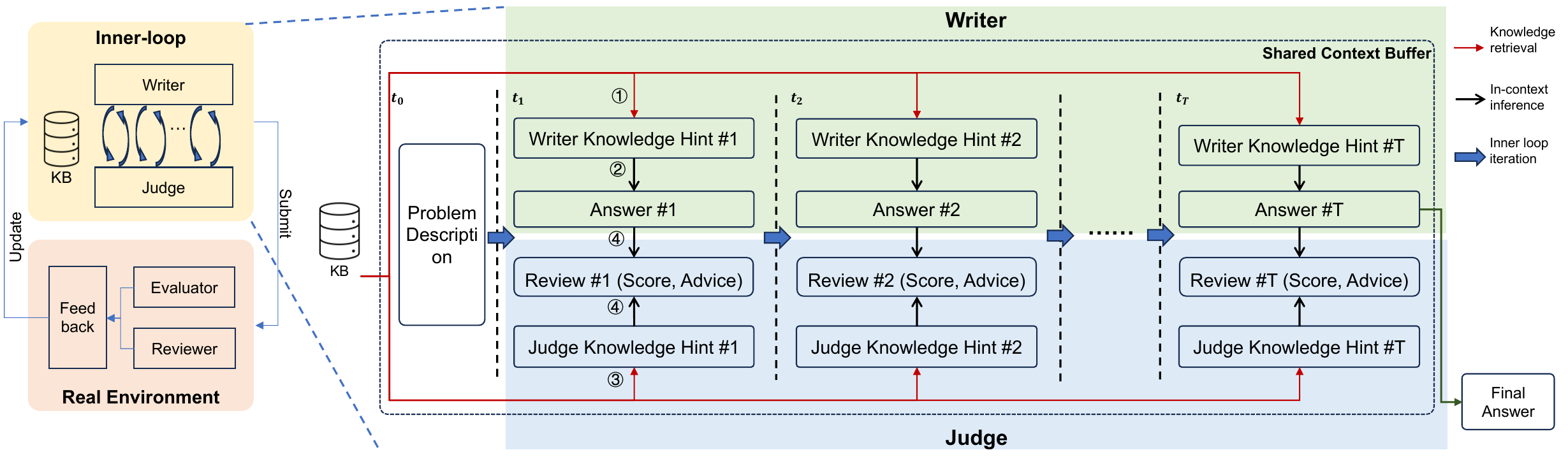}
    \caption{Our proposed framework. \textit{Left:} The outer feedback loop. The agent generates an answer to the problem, submit to the real environment for evaluation, and receives feedback. The feedback is updated to the knowledge base (KB) of the agent that serves further usage. \textit{Right:} The pipeline of the inner reasoning loop. During each loop iteration, the following things happen sequentially: \textcircled{1} The writer retrieves relevant knowledge from the KB to inform the initial answer; \textcircled{2} The writer generates an answer utilizing the retrieved knowledge; \textcircled{3} The judge retrieves relevant knowledge from KB to hint its review; \textcircled{4} The answer undergoes review, receiving a score and advice for improvement, which is then used for the next iteration. Both the writer and the judge are LLMs with specifically designed prompts and each step in the iteration uses buffer contents as context. This iterative process stops when the maximum number of steps ($T$) is reached or the score is high enough.}
    \label{fig:method}
\end{figure*}
}

\newcommand{\figureOuterLoop}{
\begin{figure}[!t]
    \centering
    \includegraphics[width=0.65\linewidth]{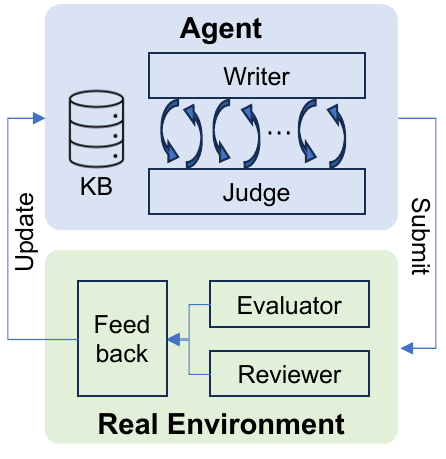}
    \caption{The outer feedback loop.}
    \label{fig:outer_loop}
\end{figure}
}

\newcommand{\figureMethodOverview}{
\begin{figure*}[!t]
    \centering
    \begin{subfigure}[b]{0.3\textwidth}
        \centering
        \includegraphics[width=\linewidth]{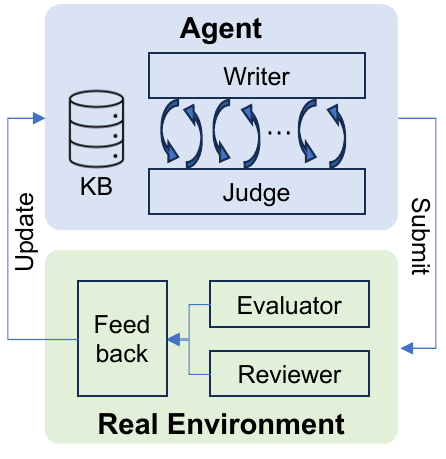}
        \caption{Overview of our proposed framework.}
        \label{fig:outer_loop}
    \end{subfigure}
    \hfill 
    \begin{subfigure}[b]{0.3\textwidth}
        \centering
        \includegraphics[width=\linewidth]{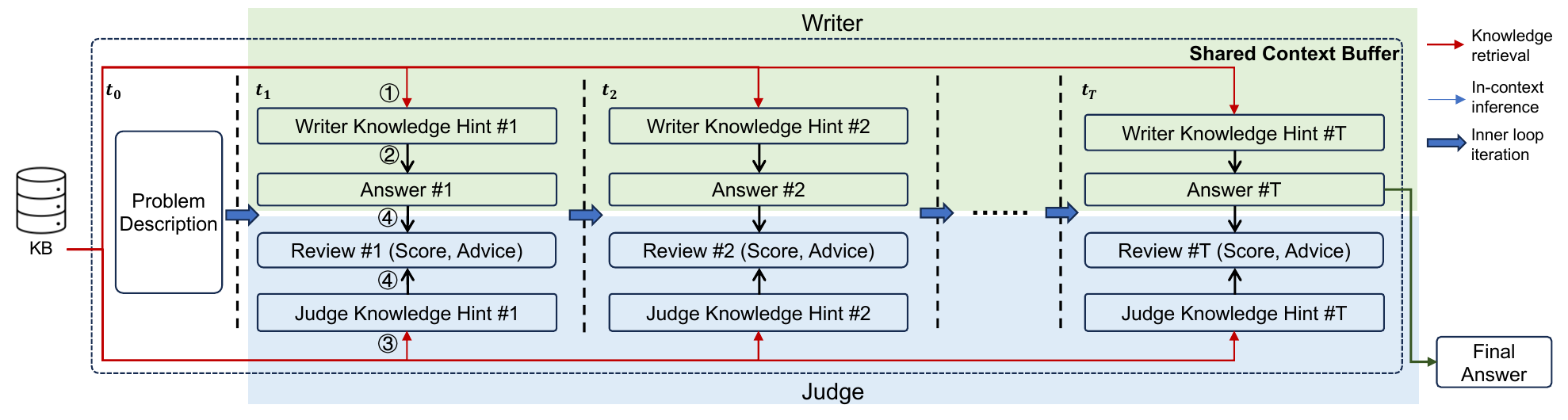}
        \caption{The inner reasoning loop.}
        \label{fig:outer_loop}
    \end{subfigure}
    \caption{Combined figure of the inner reasoning loop and the outer feedback loop.}
    \label{fig:combined_loops}
\end{figure*}
}

\newcommand{\figureProofRoute}{
\begin{figure}[!t]
    \centering
    \includegraphics[width=0.7\linewidth]{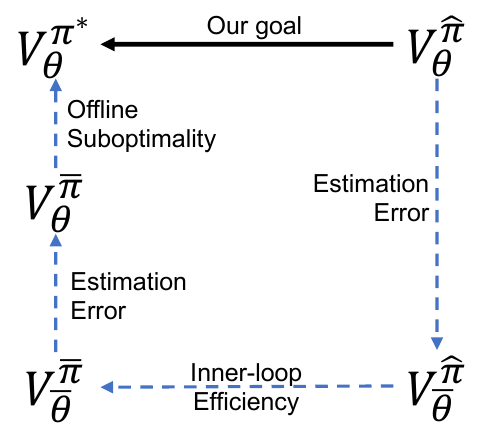}
    \caption{The route of our proof}
    \label{fig:proof_route}
\end{figure}
}

\newcommand{\figureAnalysis}{
\begin{figure*}[!t]
    \centering
    \includegraphics[width=\textwidth]{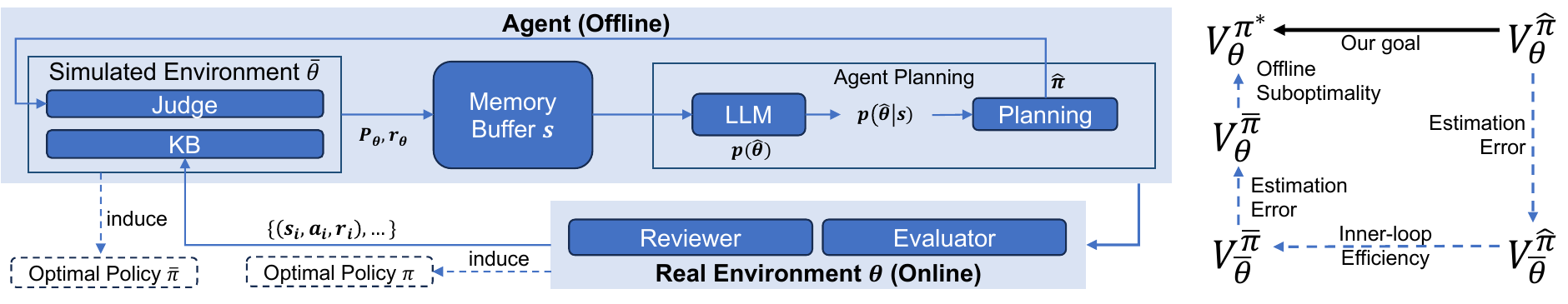}
    \caption{The analysis framework. \textit{Left:} An abstract view of the inner and outer loop, where three different environments (real environment $\theta$, simulated environment $\bar{\theta}$ and the LLM-inferred environment $\hat{\theta}$) are characterized. \textit{Right:} The chain of proofs that links the agent's actual policy $\hat{\pi}$ and the optimal policy in the real environment $\pi^*$}
    \label{fig:analysis}
\end{figure*}
}

\newcommand{\figureRelevance}{
\begin{figure*}[!t]
    \centering
    \includegraphics[width=\textwidth]{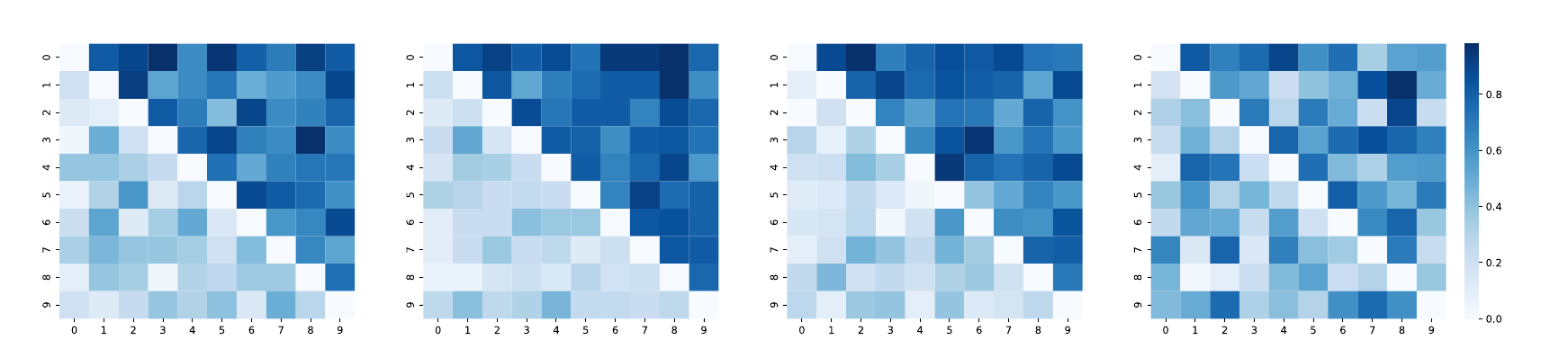}
    \caption{Trading idea relevance comparison. From left to right: with inner and outer loop; without inner loop, only outer loop;without outer loop, only inner loop; no inner or outer loop.}
    \label{fig:relevance}
\end{figure*}
}

\newcommand{\figureSelfImprovingRelevance}{
\begin{figure*}[!t]
    \centering
    \includegraphics[width=\textwidth]{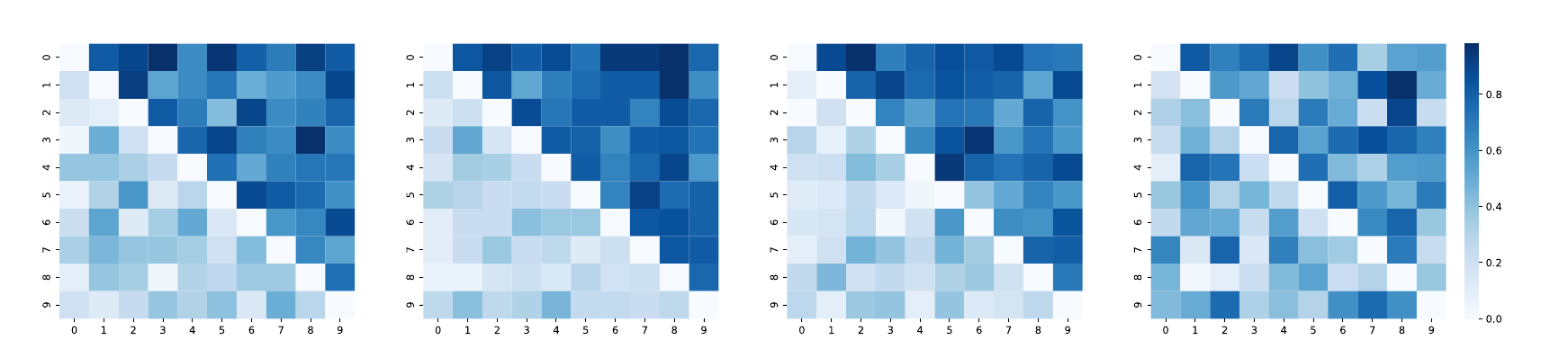}
    \caption{Self-improving code-idea relevance}
    \label{fig:relevance}
\end{figure*}
}

\newcommand{\figureSelfImprovingGrpIC}{
\begin{figure*}[!t]
    \centering
    \includegraphics[width=\textwidth]{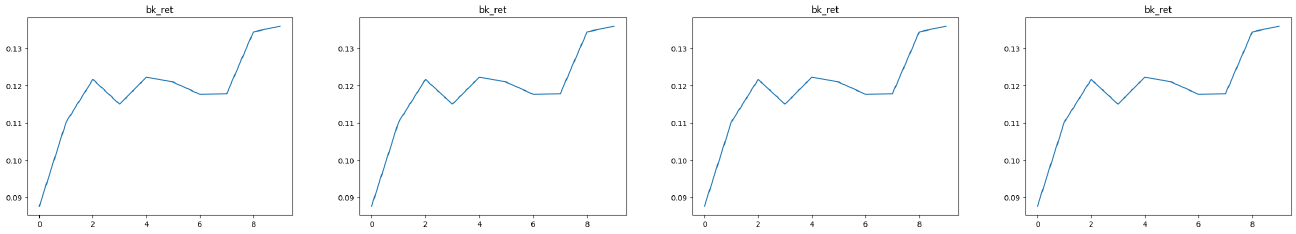}
    \caption{Self-improving group mean IC}
    \label{fig:mean_grp_ic}
\end{figure*}
}

\newcommand{\figureSingleAlphaPerf}{
\begin{figure*}[!t]
    \centering
    \includegraphics[width=\textwidth]{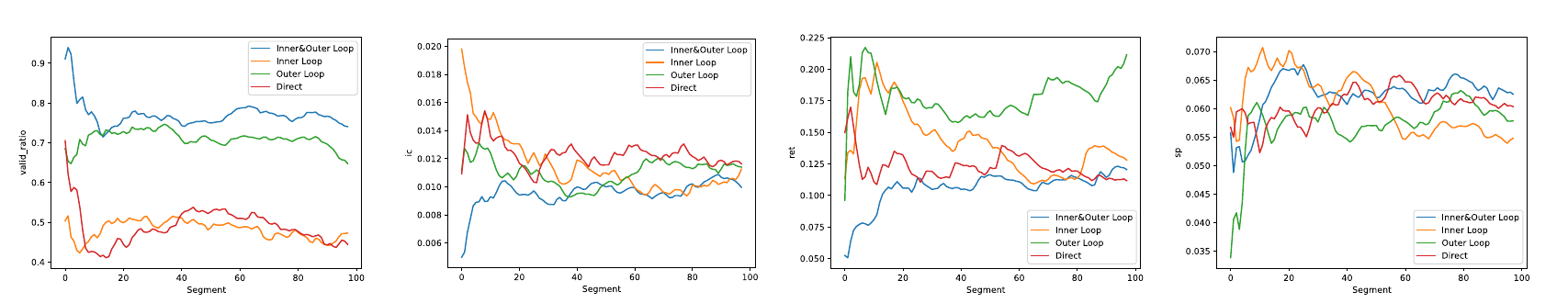}
    \caption{Evolving curve of averaged grouped single-alpha performance. From left to right: entry valid ratio, IC score, return, Sharpe ratio}
    \label{fig:single_alpha_perf}
\end{figure*}
}

\newcommand{\figureSelfImprovingMdl}{
\begin{figure}[!t]
    \centering
    \includegraphics[width=0.7\linewidth]{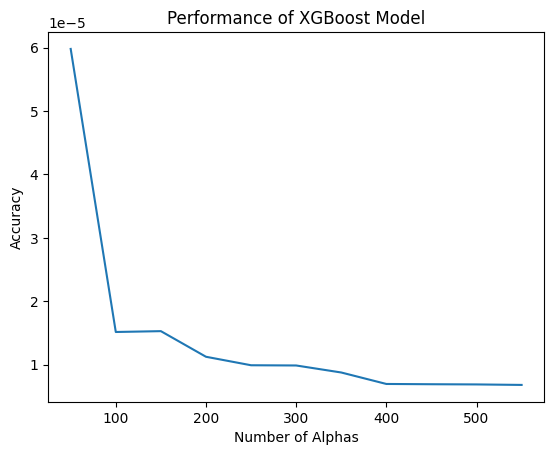}
    \caption{Increasing predictive accuracy in the trained model as the number of alpha accumulates. Y-axis indicates the MSE error}
    \label{fig:model_improve}
\end{figure}
}

\newcommand{\figureSelfImprovingGrpValid}{
\begin{figure*}[!t]
    \centering
    \includegraphics[width=\textwidth]{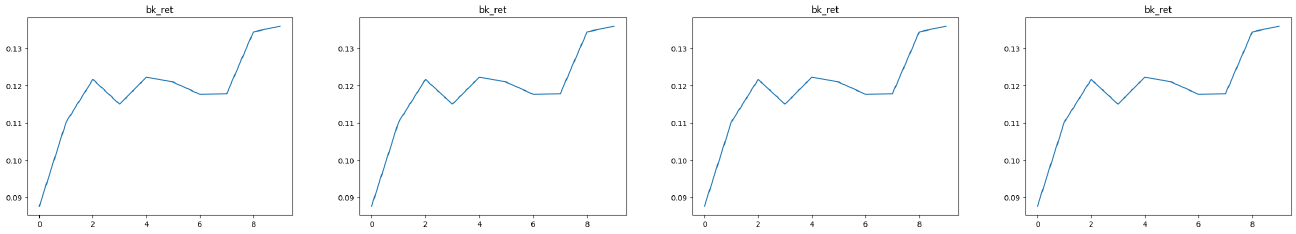}
    \caption{Self-improving group mean return}
    \label{fig:mean_grp_valid}
\end{figure*}
}

\newcommand{\figureSelfImprovingAccuModel}{
\begin{figure*}[!t]
    \centering
    \includegraphics[width=\textwidth]{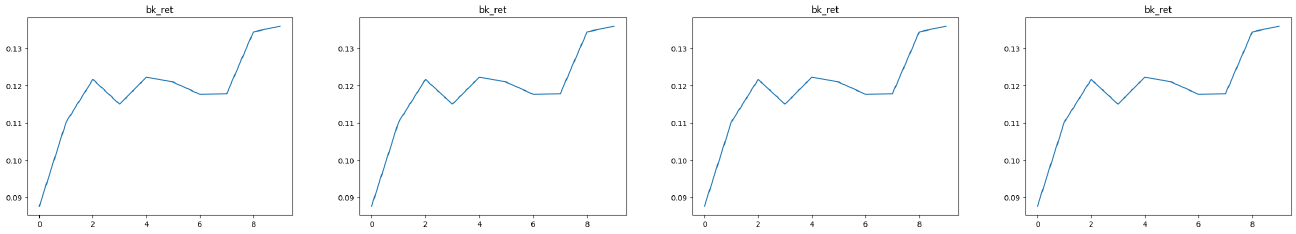}
    \caption{Self-improving group mean model performance}
    \label{fig:accu_perf}
\end{figure*}
}

\newcommand{\figureSelfImprovingGrpModel}{
\begin{figure*}[!t]
    \centering
    \includegraphics[width=\textwidth]{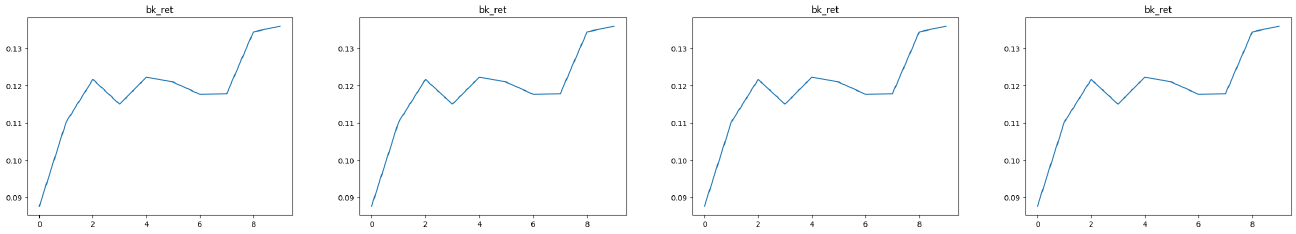}
    \caption{Self-improving accumulated model performance}
    \label{fig:grp_model_perf}
\end{figure*}
}


\newcommand{\algInnerLoop}{
\begin{algorithm}[!t]
    \caption{The inner reasoning loop}
    \label{alg:alpha_agent}
    \begin{algorithmic}[1]
        \item[\textbf{Input:}] The problem $p$, Knowledge base $\phi$, Reward threshold $\beta$, Max iteration  limit $T$.
        \item[\textbf{Output:}] An answer $\alpha$
        \STATE \texttt{Judge} $\leftarrow$ LLM prompted with judge prompt
        \STATE \texttt{Agent} $\leftarrow$ LLM prompted with agent prompt
        \STATE \texttt{ctx} $\leftarrow$ \{$p$\}
        \STATE $t \leftarrow 0$
        \REPEAT
            \STATE $K_1$ $\leftarrow$ retrieve(\texttt{Agent}, \texttt{ctx}, $\phi$)
            \STATE \texttt{ctx} $\leftarrow$ \texttt{ctx} $\cup$ \{$K_1$\}
            \STATE $\hat{\alpha} \leftarrow $ \texttt{Agent}.generate(\texttt{ctx})
            \STATE $K_2 \leftarrow $ retrieve(\texttt{Judge}, \texttt{ctx}, $\phi$)
            \STATE \texttt{ctx} $\leftarrow$ \texttt{ctx} $\cup$ \{$\hat{\alpha}, K_2$\}
            \STATE \texttt{score}, \texttt{review} $\leftarrow$ \texttt{Judge}.eval($\hat{\alpha}$, \texttt{ctx})
            \STATE $t \leftarrow t + 1$
            \STATE \texttt{ctx} $\leftarrow$ \texttt{ctx} $\cup$ \{\texttt{review}\}
        \UNTIL{$r \ge \beta$ or $t \ge T$}
        \STATE \textbf{return} $\hat{\alpha}$
    \end{algorithmic}
\end{algorithm}
}

\newcommand{\algOuterLoop}{
\begin{algorithm}[!t]
    \caption{The Outer Feedback Loop}
    \label{alg:outer_loop}
    \begin{algorithmic}[1]
        \item[\textbf{Input:}] Problem description $p$, Other inner loop params $c$, Number of outer loop iterations $K$
        \item[\textbf{Output:}] Optimized agent knowledge base $\phi$
        \STATE $\phi \leftarrow \emptyset$
        \FOR{$i = 1$ \textbf{to} $K$}
            \STATE $a$ $\leftarrow$ innerLoop($p$, $\phi$, $c$)
            \STATE $f \leftarrow$ Real-world feedback for $a$ given problem $p$
            \STATE Update $\phi$ with $a$ and $f$
        \ENDFOR
        \STATE \textbf{return} $\phi$
    \end{algorithmic}
\end{algorithm}
}

\begin{abstract}
Autonomous agents based on Large Language Models (LLMs) that devise plans and tackle real-world challenges have gained prominence.
However, tailoring these agents for specialized domains like quantitative investment remains a formidable task. 
The core challenge involves efficiently building and integrating a domain-specific knowledge base for the agent's learning process.
This paper introduces a principled framework to address this challenge, comprising a two-layer loop.
In the inner loop, the agent refines its responses by drawing from its knowledge base, while in the outer loop, these responses are tested in real-world scenarios to automatically enhance the knowledge base with new insights.
We demonstrate that our approach enables the agent to progressively approximate optimal behavior with provable efficiency.
Furthermore, we instantiate this framework through an autonomous agent for mining trading signals named QuantAgent. Empirical results showcase QuantAgent's capability in uncovering viable financial signals and enhancing the accuracy of financial forecasts.
\end{abstract}
\section{Introduction}
\label{sec:intro}

The surge of large language models (LLMs) \cite{openai_gpt-4_2023, touvron_llama_2023} has ignited significant advancements in the realm of autonomous agents \cite{weng_llm_2023}, expanding their capabilities to plan and solve complex, real-world tasks.
Applying these powerful language models to specialized areas such as quantitative investment \cite{wang_alpha-gpt_2023} brings its own set of challenges. One key issue is how to feed these models with specific knowledge from the field of finance, which would enable them to tackle complex problems like an expert in that domain.

The primary issue at hand is the source and integration of domain knowledge. In specialized domains that demand a wealth of experience and technical acumen, constructing a comprehensive knowledge base often requires intensive human effort. Meanwhile, traditional methods have relied on incorporating this knowledge into the agent's parametric memory through fine-tuning \cite{hu_lora_2021} or by referencing external databases during inference via retrieval-augmented generation techniques \cite{karpukhin_dense_2020}.

However, such approaches face significant hurdles. Crafting an extensive and accurate domain knowledge base is not only costly in terms of human labor but may also be impractical or unattainable for certain domains, such as developing a database of financial alphas \cite{tulchinsky_introduction_2019} for quantitative investment strategies.
Additionally, there is room in the current research landscape for a more principled framework that could provide a systematic assessment of knowledge integration's impact on the agent, thereby supporting the advancement of knowledge-enhanced agents.

This paper proposes a principled two-layer framework designed to autonomously develop a domain-specific knowledge base with minimal human intervention while maintaining high quality. At its core, the framework operates on a nested loop system: within the inner loop, the agent refines its responses by interacting iteratively with a simulated environment defined by an internal knowledge base. Conversely, in the outer loop, the agent's responses are evaluated against the real-world environment, which automatically generates feedback to further enrich the internal knowledge base. This iterative process propels the agent towards improved performance, ultimately allowing for the autonomous accrual of a rich knowledge base.

We bolster the framework with a theoretical analysis, demonstrating that both the inner and outer loops can efficiently converge towards an optimal solution. This convergence is substantiated by applying analytical techniques from reinforcement learning \cite{liu_reason_2023, jin_is_2021}, framing the self-improvement mechanism within the context of a Markov Decision Process. The resulting framework not only encapsulates a variety of existing self-improving methods \cite{wang_voyager_2023, romera-paredes_mathematical_2023} but is also provably efficient under certain assumptions.

The practical implementation of this framework is realized through QuantAgent, an autonomous agent tasked with mining financial signals. The agent's knowledge base comprises a collection of financial signals, each documented with its implementation details, the underlying trading idea, performance metrics, and expert reviews. This setup epitomizes the scenarios where constructing a robust knowledge base is traditionally challenging, and where evaluation—backtesting in this context—can be conducted programmatically. Our empirical findings validate QuantAgent's capacity for self-improvement, with the agent successfully generating a comprehensive signal base that enables more accurate financial forecasting.

The remainder of this paper is organized as follows: we first review related works in the field, then introduce our framework and its operational pipeline. Then we present an analysis section discussing the efficiency and practicality of our algorithm, leading to a presentation of QuantAgent as an instantiation for financial signal mining. We conclude with empirical results accompanied by a discussion of broader implications.

\section{Related Works}

\subsection{LLM-based Autonomous Agents}
Recent studies \cite{weng_llm_2023, sumers_cognitive_2023} have tried to formally describe the anatomy of an autonomous agent. Despite the differences in terms, an llm-based autonomous agent basic consists of a set of tools (where the ability of using tools can be learned \cite{schick_toolformer_2023}), a memory (be it long-term memory such as KB or short-term memory characterized as context), a planning algorithm \cite{yao_react_2022, yao_tree_2023} that directs what to do. In practice there have been many applications \cite{significant_gravitas_autogpt_2024, nakajima_yoheinakajimababyagi_2024} that resort to real-world tasks.

\subsection{Adapting LLM Agents to Domain-specific Tasks}
Adapting large language model (LLM)-based autonomous agents to domain-specific tasks involves techniques such as retrieval-augmented generation (RAG) \cite{gao_retrieval-augmented_2024} and fine-tuning \cite{wang_aligning_2023}. RAG enhances LLMs' ability to generate more accurate and contextually relevant responses by integrating external knowledge bases, proving particularly beneficial in specialized fields requiring extensive factual information \cite{sun_think--graph_2023}. Fine-tuning, on the other hand, adjusts pre-trained models on domain-specific datasets, allowing for improved performance in niche tasks by making the models more attuned to the unique lexicon and nuances of the domain. These methodologies facilitate the application of LLM agents across various sectors by significantly boosting their precision and relevance.

\subsection{Self-improving LLM agents}
Building upon the foundation of adapting LLM agents to domain-specific tasks, self-improvement emerges as a crucial next step, addressing the challenge of acquiring domain-specific knowledge, which can be expensive or difficult to obtain. These self-improvement agents iteratively learn from their environment through feedback. Recent studies have concentrated on environments including gaming \cite{wang_voyager_2023}, programming \cite{haluptzok_language_2022}, and mathematical problem-solving \cite{zhu_large_2023, romera-paredes_mathematical_2023, trinh_solving_2024}, which naturally offer the continuous feedback necessary for such learning processes. These environments not only provide rich, dynamic data for the agents to learn from but also enable the practical application of self-improvement techniques, reducing the reliance on manually curated, domain-specific datasets. This approach facilitates the development of LLM agents that can autonomously enhance their capabilities, adapt to new challenges, and refine their knowledge base over time, making self-improvement an integral component of their design.

\section{Framework}
\label{sec:method}
In this section, we elaborate on the two-layered architecture of our proposed framework.
The outer loop, depicted on the left side of Figure \ref{fig:method}, represents the interaction with the real-world environment. Here, the agent's generated output is subject to real-world evaluation, and the feedback obtained is incorporated into the agent’s knowledge base, which in turn informs subsequent iterations.
The right side of Figure \ref{fig:method} details the inner reasoning loop, where an iterative dialogue between the writer and judge components takes place. This simulated environment leverages the internal knowledge base and is where the agent's reasoning and refinement processes occur. The agent iterates through this loop, harnessing the shared context buffer for knowledge retrieval and inference until a satisfactory solution is produced or a pre-set threshold is reached.
\figureMethod
\algInnerLoop

\subsection{The Inner Reasoning Loop}
The inner loop functions as a simulated reasoning environment where an LLM or rule-based system interacts with a knowledge repository. It employs a memory buffer that initially contains the user's query and is enriched iteratively with data from the knowledge base. The agent poses queries to the knowledge base, which in turn provides pertinent information. A judge, which could be a rule set or an LLM, evaluates this information to ensure it satisfies the user's query, concluding the loop when a response is formulated or a predefined condition is met.

\subsubsection{Components}
\paragraph{Knowledge Base}
The knowledge base acts as a repository containing records of the agent's prior outputs, associated performance scores, and feedback. The agent formulates queries to extract necessary information, which are then processed into embeddings for efficient retrieval of relevant records, aiming to optimize the trade-off between performance accuracy and response variety.

\paragraph{Context Buffer}
The context buffer maintains a record of the ongoing interaction, holding all previous exchanges and information. This cumulative record ensures consistency and coherence in the agent's reasoning process, allowing past knowledge to inform future responses.

\paragraph{Writer}
The writer is responsible for constructing responses based on the data retrieved from the knowledge base. It is designed to incrementally refine its outputs, integrating feedback from the judge to enhance the quality of subsequent responses.

\paragraph{Judge}
The judge serves as an evaluator, providing feedback by scoring the writer's outputs. Its effectiveness as an assessor is contingent on the quality of the knowledge base, and it plays a critical role in calibrating the agent's outputs to improve accuracy over time.

\subsubsection{Procedure}
\paragraph{One Iteration}
Within a single iteration, the writer commences by sourcing relevant knowledge from the knowledge base. Utilizing this knowledge, the writer formulates a response. This response is then assessed by the judge, who provides a score and feedback, which is incorporated back into the context buffer to refine the writer's next response.

\paragraph{Iterative Process}
The iterative process is predicated on the accumulation of information leading to progressively improved responses. The loop is repeated, leveraging feedback to incrementally enhance the quality of responses, until a predefined performance threshold is achieved or an optimal response is determined.

\textit{Remarks:} The guiding principle behind the inner loop is that through a robust mechanism for generating responses, the writer will, with sufficient iterations, accumulate adequate information from the knowledge base to consistently satisfy the judge's criteria. This iterative enrichment is expected to guide the writer towards producing the optimal answer. A formal analysis of this convergence process is presented in Sec. \ref{sec:analysis_efficiency_inner}.

\algOuterLoop
\subsection{The Outer Feedback Loop}
The outer loop encapsulates the agent's iterative interactions with the real-world environment, where its generated outputs are evaluated and refined.

\paragraph{Environment Feedback}
The environment provides feedback in the form of performance scores and qualitative reviews, which may be generated by a sophisticated LLM. This feedback serves to inject new insights into the agent's decision-making process, with the potential to enhance future performances.

\paragraph{Knowledge Update}
Following the receipt of feedback, the knowledge base undergoes an update process. This process incorporates sanity checks to ensure the integrity and relevance of the new information. The rules for updating are crafted to maintain a comprehensive database that includes a spectrum of experiences, both successful and otherwise. This approach ensures a diverse learning context for the agent, promoting a nuanced understanding and adaptation to the real-world scenarios.

\textbf{Remarks:} A distinct contrast exists between the judge's role in the inner loop and the environment's feedback in the outer loop. The former can be characterized as providing rapid, cost-effective, albeit less precise evaluations based on a limited knowledge set. In contrast, the latter is akin to a standard of truth, often more resource-intensive but offering a higher fidelity of assessment. As the volume of iterations within the outer loop increases, the inner loop judge accumulates extensive real-world experiences, which gradually refines its capacity for delivering high-fidelity evaluations and feedback.

\subsection{Comparison with existing methods}
Many existing methods can be regarded as specific implementations of our framework. On the one hand, if we discard the outer loop, then the family of self-refinement methods \cite{madaan_self-refine_2023} can be regarded, with two LLMs acting as actor and critic either with or without a knowledge base. On the other hand, if we reduce the inner loop to a vanilla retrieval-augmented generation procedure, then many self-improving methods such as Voyager \cite{wang_voyager_2023} and FunSearch \cite{romera-paredes_mathematical_2023} can be implemented.

\section{Analysis}
\label{sec:analysis}
In this section we will analyze the efficiency (i.e. the agent algorithm can converge to a optimal solution to this problem asymptotically), and cost (in terms of both token cost and inference time cost). The motivation of our analysis is 1). understand the effect of each design component in our system 2). analyze whether it is practical to deploy it to real world.

\subsection{Efficiency}
\label{sec:analysis_efficiency}
We prove the efficiency of our agent algorithm at two levels: the inner loop level and the outer loop level. Efficiency at the inner loop indicates that the agent can reach the best possible answer to the problem given current knowledge base, and efficiency in the outer loop indicate that the performance gap between the optimal policy in simulated environment characterized by the KB and that in the real world can converge as the outer loop iteration increases. The two efficiency guarantees together forms the guarantee that the policy can converge to global optimal, indicating the efficiency of our agent algorithm.
To analyze the efficiency, we first introduce the formulation of our problem.

\figureAnalysis

\subsubsection{Formulation}
\label{sec:analysis_efficiency_formulation}
We formulate the process of generating an answer to the problem (essentially the inner loop) as a Markov Decision Process (MDP). This MDP is defined by the tuple $ (\mathcal{S}, \mathcal{A}, T, r, \gamma) $, where each component is characterized as follows:

\begin{itemize}[leftmargin=1em]
  \item \textbf{State Space} $ \mathcal{S} $: The state at any time $ t $, denoted as $ s_t $, is a composite of the original problem $ p $, and all the information contained in the shared context buffer. Initially $ s^p_{t_0} $ is an empty set, indicating the start of the reasoning process. $s_t$ that encapsulates all historical information explored up to time $t$, and we call it the information state \cite{liu_reason_2023}.
  
  \item \textbf{Action Space} $ \mathcal{A} $: Action is w.r.t. the writer since only for writer action is meaningful. Action is like generating an answer to the problem, based on current information at hand. action can also be knowledge base queries since it also changes the information state.
  
  \item \textbf{Transition Function} $ T $: Defined as $ T(s'|s, a) $, states that based on current information at hand, what will the information state change given an action. This has several implementations: 1). KB responding to a query action 2). Judge reviewing an answer action. So both the KB and the judge constitutes the transition function.
  
  \item \textbf{Reward Function} $ r $: This function, $ r(s_t) $, assigns a value to the information state $s_t$. $r$ can be either a continuous value, indicating the absolute score of an answer, or a $0/1$ value indicating whether the answer is sufficient for this problem.
  
  \item \textbf{Discount Factor} $ \gamma $: It defines the weight of future rewards in the agent's consideration, shaping the strategic depth of the reasoning process by emphasizing the importance of long-term outcomes. In practice, we use $\gamma \in (0, 1)$ to ensure that the value function $V^\pi_\theta(s)$ is bounded, which is important for analysis.
\end{itemize}

Given a policy $\pi: \mathcal{S} \mapsto \mathcal{A}$ and an environment parameterized by $\theta$, we define its value function and Q-function as follows:
\begin{equation}
    V^\pi_\theta(s) = \mathbb{E}\bigl[\sum_{t=0}^\infty \gamma^t r_\theta(s_t, a_t) | s_0=s\bigr]
\end{equation}
\begin{equation}
    Q^\pi_\theta(s, a) = \mathbb{E}\bigl[\sum_{t=0}^\infty \gamma^t r_\theta(s_t, a_t)|s_0=s, a_0=a\bigr],
\end{equation}
where the expectation is taken over $a_t \sim \pi(s_t)$ and $s_{t+1} \sim T_\theta(\cdot | s_t, a_t)$ for all $t > 0$. Here $\theta$ can be considered as knowledge base contents. Specifically, $\theta$ affects the transition function $T_\theta(\cdot | s_t, a_t)$ by determining on the query results of a given query action, and the reward function $r_\theta(s)$ is also affected by the knowledge base since the judge depend on it. The goal of the inner loop, therefore, is to learn a policy that maximizes $V^\pi_\theta(s)$ for all $s \in \mathcal{S}$ so as to find enough information for arbitrary user questions.

Our goal is to prove that our agent is provably efficient in both the inner and outer loop. Namely, the Bayesian regret of the agent
\begin{equation}\label{equ:overall_regret}
    \mathcal{R}(TK) = \mathbb{E}_{} \bigl[ \sum_{k=1}^K \sum_{t^k=1}^T V^{\pi^*}_\theta(s_t) - V^{\hat{\pi}_{t^k}}_\theta (s_t) \bigr],
\end{equation}
is sublinear in $KT$.

\subsubsection{The inner loop}
\label{sec:analysis_efficiency_inner}
Following \cite{liu_reason_2023}, we make the following assumptions on the inner loop:
\begin{assumption}\label{assump:llm_bayesian_inference}
    During the in-context inference step, the LLM writer implicitly performance Bayesian inference of the simulated environment's parameter.
\end{assumption}
Assumption \ref{assump:llm_bayesian_inference} essentially states that LLM performs implicit Bayesian inference $p(\hat{\theta}|s)$ of the environment parameter $\hat{\theta}$ given the information state $s$ and its prior knowledge $p(\hat{\theta})$ obtained via pretraining. This mechanism has also been verified in previous works \cite{xie_explanation_2021}. Based on the estimated environment parameter, we employ a planning mechanism to ensure optimality on $\hat{\theta}$:
\begin{definition}
    ($\epsilon$-Optimality \cite{liu_reason_2023}) Under an environment parametrized by $\theta$, a policy $\pi$ satisfies $\epsilon$-optimality if the following condition holds with high probability
    \begin{equation}\label{equ:eps_optim}
        \max_{s} V^{\text{PL}^*(\theta)}_\theta(s) - V^\pi_\theta(s) \le \epsilon,
    \end{equation}
\end{definition}

In this way, provable efficiency in the inner loop states that by iteratively refining its answer under the growing information and judge's feedback, 

\begin{lemma}\label{lem:rafa_result}
    The Bayesian regret of the planning agent in the inner loop is sublinear in the number of inner loop iterations $T$
\end{lemma}

\begin{proof}
    Detailed proof can be seen in \cite{liu_reason_2023, wang_principled_2023}
\end{proof}

The intuition behind this proof is that since LLM inference is implicit Bayesian inference, its posterior estimation of the simulated world parameter gets more and more accurate, with information gap eventually converges. Given the optimality of the planning algorithm, the sub-optimality depends only on the model estimation error. So as information converges, the model esitmation also converges, leading to optimal policy in the end.

Meanwhile, it is also notable that one-shot method which correspond to the close-loop solution has no theoretical guarantee of efficiency, meaning that they are not theoretically convergent. This is because the information in LLM inference does not accumulates as iteration grows. However, in other works such method also works quite well, indicating the actual information gap after 1 shot may not be that large. Moreover, other works also do not explicitly incorporate a planning mechanism while also achieving good results, and this might be because the sub-optimality might not be that severe either in real-world practice.

\subsubsection{The outer loop}
\label{sec:analysis_efficiency_outer}
Efficiency in the outer loop essentially is the result from pessimism. We first assume that the optimal policy on the knowledge base can be trained via pessimism. In this way, the knowledge base can be regarded as a offline dataset and obtaining such an optimal policy can be regarded as performing offline RL on this dataset. \cite{jin_is_2021} states that the performance gap between the best policy learned offline and the best policy online can be bounded by the information gap, if the offline policy is learnt with pessimism. Our analysis will assume this and based on this to get our efficiency proof.

\begin{assumption}\label{assump:simu_optim_pessimism}
    Given a simulated environment parametrized by $\Bar{\theta}$ characterized by an offline knowledge base, the optimal policy $\Bar{\pi}$ can be obtained pessimistic value iteration (PEVI) \cite{jin_is_2021}.
\end{assumption}

\begin{lemma}\label{lem:offline_subo_bound}
    (Efficiency of Pessimism) Under Assumption \ref{assump:simu_optim_pessimism}, the performance gap between $\Bar{\pi}$ and $\pi^*$ in the real environment is bounded by the intrinsic uncertainty caused by insufficient coverage of the knowledge base.
\end{lemma}
\begin{proof}
    The detailed proof can be found in Section 4 of \cite{jin_is_2021}.
\end{proof}

Lemma \ref{lem:offline_subo_bound} guarantees that under pessimism, the sub-optimality is only related to the estiamtion inaccuracy in model parameters between the real environment and the simulated one characterized by the knowledge base. And as the KB accumulates more and more information about the real environment, making the simulated one closer to the real one, the optimal policy trained using pessimism on this simulated environment should also have a converging performance gap. This theoreical results bridges the optimal policy in the real env and simu env, which together with the inner loop efficiency links together the gap between the policy learned by the agent and the optimal one in the real env.

\subsubsection{Overall result}
\label{sec:analysis_efficiency_overall}
Combining provable efficiency in the inner and outer loop, we can bridge the overall performance gap 

\begin{theorem}\label{thm:overall_result}
    The Bayesian regret of the LLM agent in the real environment, $\mathcal{R}(TK)$ in Eq. \ref{equ:overall_regret}, is sublinear in KT.
\end{theorem}
\begin{proof}
    We give a sketch of the proof here. Detailed proof of each term can be found in \cite{liu_reason_2023, jin_is_2021}.
    The performance gap can be essentially decomposed as follows:
    \begin{align}
    V^{\pi^*}_\theta(s_t) - V^{\hat{\pi}}_\theta (s_t) &= 
        \underbrace{V^{\pi^*}_\theta (s_t) - V^{\bar{\pi}}_\theta (s_t)}_{\text{Term A}} \nonumber \\
        \quad + \underbrace{V^{\bar{\pi}}_{\theta} (s_t) - V^{\bar{\pi}}_{\mathclap{\bar{\theta}}} (s_t)}_{\text{Term B}} & + \underbrace{V^{\bar{\pi}}_{\mathclap{\bar{\theta}}}(s_t) - V^{\hat{\pi}}_{\mathclap{\bar{\theta}}}(s_t)}_{\text{Term C}} + \underbrace{V^{\hat{\pi}}_{\mathclap{\bar{\theta}}}(s_t) - V^{\hat{\pi}}_\theta(s_t)}_{\text{Term D}}  \label{eq:performance_gap}
    \end{align}
    When taking sum over $T$ and $K$, Term A, B, given pessimism assumption (Assumption \ref{assump:simu_optim_pessimism}) can be bounded by the intrinsic uncertainty in the offline dataset (lemma \ref{lem:offline_subo_bound}), which is sublinear in $K$ under modest assumptions (e.g. linear MDP) of the underlying MDP. Similarly, for Term D, it is also bounded by the information gap which is sublinear in $K$ in the same case. For Term C, according to Lemma \ref{lem:rafa_result}, it is sublinear in $T$. Hence, the LHS, when taken sum in Eq. \ref{equ:overall_regret}, is sublinear in $KT$.
\end{proof}


\subsection{Cost Analysis}
\paragraph{Token Cost}
In the self-improvement phase, the inner loop incurs a token cost that escalates with the square of the planning horizon $H$ and the number of close-loop interaction rounds $T$. This yields a computational complexity of $O(T^2H)$ for a single outer loop iteration. When considering multiple rounds of outer loop iterations, necessary for the agent's self-improvement, the cumulative token cost across $K$ such iterations scales to $O(KT^2H)$. During the inference phase, where the outer loop is not executed, the token cost remains at $O(T^2H)$, assuming that the complexity of generating responses does not change significantly from training to inference.

\paragraph{Time Cost}
The time cost for each iteration within the outer loop is influenced by both the planning horizon and a constant factor representing the baseline computational overhead. Assuming the computational time for key operations is constant and each value computation incurs one unit of time, a single outer loop iteration requires $T(H + \text{Const})$ units of time. Consequently, the entire training phase, encompassing $K$ iterations of the outer loop, demands $O(KTH)$ time. In the inference phase, devoid of the outer loop, the time complexity is reduced to $O(HT)$.

\section{Experiment}
In this section we introduce the experiment settings in our paper. Our experiment is conducted on a domain-specific scenario, so we will introduce tasks background, and then the most important setups in our experiment. For detailed information see appendix.

\subsection{Background}
Our task is financial signal mining. Financial signals are predictive signals computed from financial market data that can be used for financial prediction.
Designing financial signals requires good market understanding, code implementation skills, and mathematical skills such as numerical analysis.
The primary goal of this task is to get good financial signals that achieves high predictive power. Given a trading idea, we would also like the agent to improve its  We would also like to get a diversified set of signals on which we can build machine learning models to combine them to get better predictions.

\subsection{Problem Setup}
The general goal is to generate good financial signals evaluated by our measurements. To achieve this goal, at each outer loop iteration, the agent is given a trading idea, and then asked to generate financial signals that both reflects this trading idea while also achieves good performance (by taking other considerations). The trading idea is sampled from a distribution, in our experiment generated by another LLM. The implementation of a financial signal is essentially a piece of code (or equivalently a function) with pre-defined schema. Similar to FunSearch, we define a template shown in the appendix and let the LLM agent implement the function codes for that alpha.

\paragraph{Dataset} We selected a universe of 500 stocks on Chinese A-share market. The time range of these data is taken in year 2023. The basic market data used for computing financial signals contains the volume and price day collect at market close at each day.

\paragraph{Foundation LLM} We selected gpt-4-0125-preview as the version for our foundation model.


\subsection{Evaluation protocol}
\paragraph{Predictive Performance}
To assess the predictive power of the financial signals, we calculate the Information Coefficient (IC) with respect to future returns. IC is the Pearson correlation computed for each cross-section (all stocks at a single time point) and then averaged across all time points. To determine the efficacy of the Knowledge Base (KB) as a foundational dataset for signal generation, we examine the IC of signals produced by a machine learning model, in this case, XGBoost regression trees, using the KB as the feature input. This evaluates the KB's potential as a reliable source for creating predictive models in finance. Additionally, we analyze the signal's Sharpe ratio to evaluate its ability to generate satisfactory investment returns.

\paragraph{Signal Quality}
For numerical quality, we look at the number of valid and unique entities the signals pertain to, ensuring they are capable of differentiating stocks. In terms of trading idea similarity, we measure how well the signals correspond to their underlying trading ideas. This relevance is critical for the potential use of LLMs as investment research assistants. To quantify this, we conduct pairwise comparisons where an LLM evaluates two signals generated from the same trading idea by different agents. The LLM selects the more accurate signal, and from these pairwise comparisons, we construct a win-rate matrix. The aggregated results from this matrix form a leaderboard that ranks the agents on their ability to produce signals that are true to the original trading concepts. This ranking offers insights into the proficiency of different agents in capturing the essence of trading ideas within their signal implementations.
\section{Results}
\label{sec:results}

This section presents experimental results corroborating our theoretical findings. We highlight the main outcomes indicative of the agent's self-improvement capabilities and analyze the effect of inner and outer loop via different experiments.

\subsection{Evidence of Self-Improvement}
\label{subsec:evidence_of_self_improvement}
Figure \ref{fig:model_improve} shows the consistence performance improvement in terms of both training on cumulative signals and training on individual segments of alphas. The increasing predictive accuracy shows that, as the agent self-improves, it can produce stronger signals for better predictive power. Note that 

\figureSingleAlphaPerf

\figureRelevance
\figureSelfImprovingMdl
\subsection{Improving Alpha Relevance}
\label{subsec:improving_alpha_relevance}
To measure signal quality, we sorted the signals generated throughout the whole process according to their generated iterations, and segment them into equally-numbered groups for comparison in terms of accurately conveying the underlying trading idea, judge by GPT-4. Figure \ref{fig:relevance} shows the distribution of the win-rate. The pattern of accumulating winning rates to the upper right corner indicates that as the model evolves, the agent gains better skill in writing good-quality signals, verifying the effectiveness of self-improving. Notably, the diminishing pattern in the rightmost matrix demonstrate the effectiveness of both our inner and outer loop.

\subsection{Effect of the Inner and Outer Loop}
\label{subsec:effect_outer_loop}
Figure \ref{fig:single_alpha_perf} demonstrates the evolution of single-alpha attributes across the whole process. While it is not apparent to see significant performance differences, the trend of the blue curve demonstrates the effectivenss of self-improvement in all metrics.

\section{Discussion}
\label{sec:discussion}

Our research illustrates the efficacy of autonomous agents, particularly in quantitative investment, showcasing their capability for self-improvement and adaptation through simulated environments. This advancement not only highlights their potential as tools for financial analysis and risk management but also opens the door to their application across various complex domains beyond finance, such as healthcare and logistics, by customizing the knowledge base and feedback mechanisms. Despite these promising developments, challenges such as the dependency on knowledge base quality and the need for computational optimization remain. Future endeavors will aim to enhance the agents' learning efficiency, diversify their application scenarios, and explore real-time adaptation to dynamic environments. Ultimately, these efforts underscore the vast potential of LLM-based agents to revolutionize decision-making processes across industries, marking a significant stride towards the era of intelligent automation.

\section{Broader impact}
This paper presents work with the goal of advancing the research direction of building self-improving autonomous LLM agents. There are many potential societal consequences of our work, none of which we feel must be specifically highlighted here.

\bibliographystyle{icml2024}
\bibliography{references}

\appendix
\onecolumn

\section{Detailed Problem Backgrounds}
By discerning the patterns of market changes to write trading signals that can capture and profit from these patterns, in the field of quantitative finance, we typically refer to such trading signals or characteristics that have predictive power over future markets as alphas. Usually, one writes a program to calculate historical market data to implement the core logic and ideas behind this trading signal, and then validates through calculation and backtesting whether the trading signal can be profitable. Moreover, a single alpha signal is unlikely to cover all market characteristics. To more accurately predict future markets, a set of trading signals with low correlation to each other, covering various market dimensions, is needed. Combining these signals through model construction allows for an accurate market forecast.

Therefore, an excellent quantitative finance researcher needs to be able to translate their market insights effectively into a program that constructs these trading signals. There are two challenges here: on one hand, how to translate your observations and understanding of the market, which are often described in natural language, into a computer program accurately and fittingly; on the other hand, how to build a high-quality, predictive trading signal, which may involve skills and experience in constructing quality trading signals and programs, numerical processing, and refinement. You might need to try combining other signals with yours to make it more effective at predicting the market and profiting.

A human researcher often requires many years of experience and a high level of talent to perform this job well, and the efficiency is not high—it is difficult to rapidly build a rich set of alphas to enhance market prediction in the short term. A rich and diverse library of alphas with strong predictive capabilities usually requires years of accumulation.

After LLMs demonstrated their understanding of natural language and ability to write programs, some works, such as Alpha-GPT, have attempted to use large models to generate trading signal programs by understanding natural language trading ideas. However, there are technical challenges involved, among which the most significant is that to generate better signals that more closely match the trading ideas, the large model needs knowledge enhancement.

\subsection{An example alpha}
\begin{lstlisting}
from idea_factor import Factor
import pandas as pd
import numpy as np

class VolatilityBreakoutSignal(Factor):
    def __init__(self):
        self.name = "VolatilityBreakoutSignal"
        self.window_length = 14
        self.inputs = ["high", "low", "pre_close"]
    
    def calc(self, data):
        # Calculate the components of the True Range (TR)
        tr = np.maximum(data['high'] - data['low'], 
                        np.maximum(abs(data['high'] - data['pre_close'].shift(1)), 
                                   abs(data['low'] - data['pre_close'].shift(1))))
        
        # Calculate the Average True Range (ATR) using a rolling mean of the True Range
        atr = tr.rolling(window=self.window_length).mean()

        # Define breakout threshold as 1.5 times the ATR from the previous day's high
        breakout_threshold = data['high'].shift(1) + 1.5 * atr

        # Generate the signal indicating a volatility breakout
        signal = (data['high'] > breakout_threshold) * 1

        # Normalize the signal using the distance from the breakout threshold divided by ATR
        # to make it proportional to the size of the breakout
        normalized_signal = ((data['high'] - breakout_threshold) / atr).where(signal == 1, 0)

        # Ensure all values are non-negative, and fill any NaN values with 0
        normalized_signal = normalized_signal.clip(lower=0).fillna(0)

        return normalized_signal
\end{lstlisting}

\section{Detailed Experimental Settings}
\subsection{The financial signal mining runtime environment}
We use idea-factor a framework that we developed to define trading signal / alpha

In the idea-factor framework, trading signals are implemented by defining a new class that inherits from a base class named Factor. This base class handles technical aspects of signal calculation and maintenance.

For a researcher to develop a trading signal, all you should care about is following:

1. Defining Member Variables: The class must have three member variables:
    - name: The signal's name, which should also be reflected in the class name.
    - window-length: The historical data window length needed for each signal calculation.
    - inputs: Specifies the base data required for the signal calculation.
    
2. Implementing the calc Function: This is where the trading signal's logic is realized. The framework uses a loop to call this function, allowing the calculation of the trading signal for historical periods. You can use various Python computational libraries, like scipy, to enhance your calculations. Additionally, you can define multiple auxiliary member functions within the class for complex calculations, which can be utilized within the calc function.
    - Input Parameters : a dictionary called data is passed as an input parameter, containing base data mapped from inputs. Each key corresponds to an input variable, and the value is a pandas.DataFrame with rows as dates and columns as stocks.
    - Return Value: The function should return a pandas.DataFrame, with columns representing the universe of stocks and rows indexed by dates. The idea-factor framework will then use the last row of this DataFrame as the signal values for the current period.
    
\subsection{Base data}
In this experiment, we provide the following data for creating a trading signals. 
\begin{itemize}
    \item \textbf{close}: Closing Price
    \item \textbf{high}: Highest Price
    \item \textbf{low}: Lowest Price
    \item \textbf{open}: Opening Price
    \item \textbf{pre-close}: Previous Day's Closing Price
    \item \textbf{return}: The rate of change of today's closing price relative to the previous day's closing price
    \item \textbf{shares}: Total Shares Outstanding
    \item \textbf{tradenum}: Number of Trades
    \item \textbf{turnover}: Turnover Rate
    \item \textbf{value}: Transaction Amount
    \item \textbf{volume}: Trading Volume
    \item \textbf{vwap}: Volume-Weighted Average Price
\end{itemize}

\section{Detailed Experimental Results}
\subsection{A self-improving example of inner cycle}
Trading Idea: I observed that three consecutive candlesticks each closing higher than the preceding one. This pattern suggests strong buying pressure in the market. Please implement a trading signal based on this observation

Trading Signal implementation 1st version:

\begin{lstlisting}
class ThreeSoldierSignal(Factor):
    def __init__(self):
        self.name = "Three Advancing White Soldiers"
        # Assuming a 3-day pattern, extend window to capture necessary data for calculation
        self.window_length = 4  # Extra day to account for comparison with previous conditions
        self.inputs = ["close", "open"]

    def _check_pattern(self, opens, closes):
        """
        Check if the last three days form the (Three Advancing White Soldiers) pattern.

        """
        # Criteria for this pattern:
        # 1. Three consecutive days with higher closes than opens
        higher_close_than_open = (closes.shift(-1) > opens.shift(-1)) & (closes.shift(-2) > opens.shift(-2)) & (closes > opens)

        # 2. Each day closes higher than the previous day
        consecutive_higher_closes = (closes > closes.shift(-1)) & (closes.shift(-1) > closes.shift(-2))

        # 3. The opening of each day is within the body of the previous day
        # Ignoring the gap between days as it may not be strictly applicable in all markets
        open_within_previous_body = (opens.shift(-1) > opens.shift(-2)) & (opens.shift(-1) < closes.shift(-2)) & \
                                    (opens > opens.shift(-1)) & (opens < closes.shift(-1))

        # Combine conditions, if all true pattern detected
        if higher_close_than_open.iloc[-1] and consecutive_higher_closes.iloc[-1] and open_within_previous_body.iloc[-1]:
            return 1
        else:
            return 0
            
    def calc(self, data):
        # Initialize signal DataFrame
        signal = pd.DataFrame(index=data['close'].index, columns=data['close'].columns)

        # Iterate over columns (stocks) and calculate signal
        for stock in data['close'].columns:
            combined_signal = data['open'][stock].rolling(window=self.window_length).apply(
                lambda x: self._check_pattern(x, data['close'][stock]), raw=False
            )
            signal[stock] = combined_signal.fillna(0)  # Fill NaNs with 0
            # Note: Assumes the pattern's absence on days where window_length is not met

        return signal
\end{lstlisting}

Rejected by mentor agent with the following suggestions:

1. Refine the signal calculation to better distinguish between different strengths of the pattern. Currently, it only returns a binary output (1 for pattern detected, 0 otherwise), which lacks nuance for quantitative models.

2. Consider adjusting the rolling window logic to ensure that only the necessary number of days (exactly 3) are evaluated for the pattern, to improve computational efficiency

3. Incorporate an additional condition or metric to evaluate the strength and reliability of the pattern, such as volume or the size of the candlesticks, to provide a more informative and differentiated signal.

Trading Signal implementation 2nd version:

\begin{lstlisting}
class ImprovedThreeSoldierSignal(Factor):
    def __init__(self):
        self.name = "ImprovedThreeSoldierSignal"
        # Adjusting window length to exactly 3 days as required for the pattern
        self.window_length = 3  # Only the 3 days relevant for the pattern
        self.inputs = ["close", "open", "volume"]  # Including volume as a condition for improved signal strength

    def _check_pattern(self, opens, closes, volumes):
        """
        Check the last three days for the Three Advancing White Soldiers pattern 
        and quantify its strength.
        """
        # Checking basic pattern criteria
        higher_close_than_open = (closes.shift(-2) > opens.shift(-2)) & (closes.shift(-1) > opens.shift(-1)) & (closes > opens)
        consecutive_higher_closes = (closes.shift(-2) < closes.shift(-1)) & (closes.shift(-1) < closes)

        # Additional strength evaluations
        increasing_volume = (volumes.shift(-2) < volumes.shift(-1)) & (volumes.shift(-1) < volumes)

        # Calculate signal strength based on normalized candlestick body size and volume increase
        if higher_close_than_open and consecutive_higher_closes:
            pattern_strength = (closes - opens).mean() / (closes.max() - opens.min())
            volume_strength = increasing_volume.mean()
            overall_strength = (pattern_strength + volume_strength) / 2
            return overall_strength
        else:
            return 0

    def calc(self, data):
        # Initialize signal DataFrame
        signal = pd.DataFrame(index=data['close'].index, columns=data['close'].columns)

        # Iterate over columns (stocks) and calculate the improved signal
        for stock in data['close'].columns:
            combined_signal = data['open'][stock].rolling(window=self.window_length).apply(
                lambda x: self._check_pattern(x, data['close'][stock], data['volume'][stock]), raw=False
            )
            signal[stock] = combined_signal.fillna(0)  # Fill NaNs with 0
            # Note: Assumes the pattern's absence on days where window_length is not met

        return signal
\end{lstlisting}

Rejected by mentor agent with the following suggestions:

1. The '-check-pattern' function returns 0 when the pattern does not match, which causes a lack of differentiation between the absence of the pattern and weak patterns. Consider returning a negative value for non-patterns to differentiate from weak patterns.
2.  In the '-check-pattern' function, 'higher-close-than-open' and 'consecutive-higher-closes' are boolean conditions. Their calculation does not account for vectorized operations correctly because the 'if' condition expects a singular truth value. Use aggregate functions (.all()) to handle series of booleans."
3. The calculation of 'pattern-strength' and 'volume-strength' might result in unbound values, especially 'volume-strength' as 'increasing-volume.mean()' can be misleading because 'increasing-volume' is a boolean condition. Refine strength calculations to ensure meaningful signal values

Trading Signal implementation 3rd version:

\begin{lstlisting}
class ThreeSoldierSignalV3(Factor):
    def __init__(self):
         self.name = "ThreeSoldierSignalV3"
         self.window_length = 3
         self.inputs =["close", "open", "volume"]
         
    def _calculate_pattern_strength(self, opens, closes, volumes):
        """
        Calculates the strength of the pattern based on closing prices,opening prices,and volume
        """
        # Criteria evaluation
        higher_close_than_open = (closes > opens).all()
        consecutive_higher_closes =(closes.diff().iloc[1:] > 0).all()
        volume_increasing = (volumes.diff().iloc[1:] > 0).all()
        
        # Signal strength calculation criteria
        if higher_close_than_open and consecutive_higher_closes:
            close_open_ratio = ((closes-opens) / opens.sum()).sum()
            volume_ratio = np.log(volumes.iloc[-1] / volumes.iloc[0]) if volume_increasing else 8
            # Adjusted overall strength based on magnitude of movements and volume changes
            overall_strength=(close_open_ratio * 0.7) + (volume_ratio * 0.3)
            # Ensuring the strength lies between -1 and 1
            overall_strength = max(min(overall_strength, 1), -1)
            return overall_strength
        else:
            # Differentiating from weak patterns by returning a slightly negative value
            return -0.1 # Indicates absence of pattern
    def calc(self, data):
        # Initialize signal DataFrame
        signal = pd.DataFrame(index=data['close'].index, columns=data['close'].columns)
        # Iterate over columns(stocks)and calculate the finalsignal
        for stock in data['close'].columns:
            combined_signal = data['open'[stock].rolling(window=self.window_length).apply(
                lambda x: self._calculate_pattern_strength(x, data['close'][stock], data['volume'][stock]), 
                raw=false
            )     
            signal[stock]= combined_signal.fillna(-0.1)# Filling NaNs with -0.1 to differentiate from absence of pattern
            
        return signal
    \end{minted}
\end{lstlisting}

mentor agent passed this implementation with the final review:
The implementation of the ThreeSoldierSignalV3 trading signal is quite sound, capturing the essence of the candlestick pattern of the trading idea. Its logic analyzes a three-day window for each stock, focusing on whether the close prices are higher than the open prices, if the closing prices are consecutively higher, and if trading volume is increasing, which are crucial factors for identifying this bullish pattern. The unique approach of calculating the pattern's strength based on the magnitude of price changes and volume variations, then balancing these factors to determine an overall strength score between -1 and 1, effectively captures the signal's intent. The use of logarithmic scaling for volume change and the thoughtful consideration of filling NaN values with -0.1 to signify the absence of the pattern demonstrates a careful and innovative handling of the signal's nuances

\end{document}